\documentclass[lettersize,journal]{IEEEtran}

\usepackage{amsmath,amssymb,amsfonts}
\usepackage{bm}

\usepackage{algorithm}
\usepackage{algpseudocode}

\usepackage[caption=false,font=normalsize,labelfont=sf,textfont=sf]{subfig}
\usepackage{graphicx}
\usepackage{array}
\usepackage{stfloats}

\usepackage{cite} 

\usepackage{hyperref}
\usepackage{cleveref}

\usepackage{amsthm}
\theoremstyle{plain}
\newtheorem{theorem}{Theorem}

\newtheorem{corollary}{Corollary}
\theoremstyle{definition}

\newtheorem{assumption}{Assumption}
\theoremstyle{remark}

\usepackage[dvipsnames]{xcolor}

\usepackage{bbm}             

\usepackage{tikz} 
\usetikzlibrary{arrows.meta, positioning, shapes.misc, calc, shapes.geometric, fit}

\usepackage{booktabs}  
\usepackage{cuted}

\usepackage{makecell}
\usepackage{multirow}

\author{
Jinlun Zhang, Haoneng Huang, Zishu Zhan, and Chunquan Ou%
\thanks{Corresponding authors: Zishu Zhan (e-mail: treer0927@smu.edu.cn) and Chunquan Ou (e-mail: ouchunquan@hotmail.com).}%
\thanks{Jinlun Zhang, Haoneng Huang, Zishu Zhan, and Chun-quan Ou are with the State Key Laboratory of Multi-organ Injury Prevention and Treatment, Department of Biostatistics, Guangdong Provincial Key Laboratory of Tropical Disease Research, School of Public Health, Southern Medical University, Guangzhou, China.}%
}

\begin{document}

\title{Distributional Causal Mediation via Conditional Generative Modeling}

\maketitle

\begin{abstract}
Causal mediation analysis characterizes how an exposure or treatment affects an outcome through intermediate variables. However, most existing mediation methods target low dimensional summaries, especially mean effects, and may therefore miss changes in dispersion, tail behavior, modality, and other features of the outcome distribution. We propose Distributional Causal Mediation Analysis (DCMA), a conditional generative framework for interventional mediation analysis at the level of interventional outcome distributions with multiple mediators. DCMA defines total, direct, indirect, and path-specific interventional estimands as user specified functionals of pairs of interventional outcome distributions, encompassing mean, quantile, exceedance risk, and discrepancy based summaries such as the energy distance and Wasserstein distance. We establish identification formulas under standard interventional mediation assumptions, representing the target interventional outcome distributions as functionals of the observed joint mediator distribution \(P(\bm M\mid A,\bm Z)\), the conditional outcome distribution \(P(Y\mid A,\bm M,\bm Z)\), and the covariate distribution. DCMA estimates these conditional distributions using noise driven conditional generators and reconstructs the target interventional outcome distributions by Monte Carlo forward simulation, enabling flexible estimation without explicit likelihood specification. We further derive a structural error decomposition showing how errors in the learned conditional mediator and outcome distributions propagate to the reconstructed interventional outcome distributions. Synthetic and semi-synthetic experiments with known interventional targets show that DCMA recovers mediation patterns involving bimodality, dispersion changes, tail risk, and mediator-specific distributional patterns. A NHANES liver elastography application illustrates distributional mediation analysis in observational data.
\end{abstract}


\section{Introduction}\label{sec:intro}

\IEEEPARstart{C}{ausal} mediation analysis aims to characterize how an exposure or treatment $A$ affects an outcome $Y$ through mediators $\bm M$~\cite{robins1992identifiability,imai2010general, vanderweele2014mediation}. Depending on the scientific objective, mediation analysis can be studied from an explanatory perspective, which decomposes the total effect into direct and indirect pathways, or from an interventional perspective, which defines effects through hypothetical interventions on the exposure and mediator distributions~\cite{nguyen2021clarifying}. The interventional mediation framework is particularly well suited to policy relevant \textit{what-if} questions.

However, most existing approaches summarize effects using low dimensional contrasts, especially mean differences. Such summaries can be insufficient when the treatment changes the shape of the outcome distribution, motivating distributional causal targets beyond averages~\cite{kennedy2023semiparametric,kim2018causal}. For example, a tracking program may leave the class average unchanged while widening the gap between high and low performing students, resulting in a bimodal distribution. Beyond detecting such a distributional change, one may further ask how it can be attributed to the direct pathway and to specific mediator pathways through teacher support and instructional resources.

This paper develops Distributional Causal Mediation Analysis (DCMA), a framework for interventional mediation analysis at the level of interventional outcome distributions. DCMA reconstructs the interventional outcome distributions associated with total, direct, indirect, and mediator-specific pathways, thereby enabling pathway-specific evaluation of location, dispersion, tail behavior, modality, and global distributional changes. We define distributional interventional total effects (ITEs), interventional direct effects (IDEs), interventional indirect effects (IIEs), and interventional path-specific effects (IPSEs) as user specified functionals of pairs of interventional outcome distributions. This formulation includes classical mean effects as a special case and also accommodates quantile effects, exceedance risk effects, and global distributional discrepancies such as the energy distance (ED)~\cite{rizzo2016energy} and the Wasserstein distance~\cite{villani2009optimal}.

Building on the standard identification formula for interventional mediation effects~\cite{vanderweele2014effect,vansteelandt2017interventional}, we move from expectations of potential outcomes to interventional outcome distributions. The resulting representation depends only on two observed conditional distributions: the mediator distribution \(P(\bm M\mid A,\bm Z)\) and the outcome distribution \(P(Y\mid A,\bm M,\bm Z)\). DCMA estimates these distributions using two noise driven conditional generators, namely a joint mediator generator and an outcome generator. It then samples mediators and outcomes from the fitted generators under the target intervention regime and reconstructs the corresponding interventional outcome distributions by Monte Carlo forward simulation. We analyze this reconstruction step through a structural error decomposition that separates mediator stage and outcome stage error components from the learned conditional distributions. The empirical studies evaluate recovery of interventional mediation estimands in synthetic and semi-synthetic settings with known ground truth and illustrate mean, threshold risk, and distributional mediation summaries in observational data.

The main contributions are summarized as follows:
\begin{itemize}
\item We introduce distributional interventional mediation estimands defined through interventional outcome distributions, extending mean based mediation analysis to quantile, exceedance risk, and discrepancy based summaries.

\item We develop a generative mediation framework that learns conditional mediator and outcome distributions with noise driven conditional generators and reconstructs direct, indirect, and mediator-specific interventional outcome distributions through Monte Carlo simulation.

\item We provide distribution level identification formulas and structural error bounds that quantify how mediator and outcome stage conditional distribution errors propagate to the reconstructed interventional outcome distributions.
\end{itemize}

\section{Related Work}\label{sec:related}

\subsection{Interventional Mediation Analysis}

Interventional mediation was formally introduced to enable effect decomposition in settings with exposure induced mediator--outcome confounding, where natural mediation effects may be non-identifiable~\cite{vanderweele2014effect}. By defining effects through interventions on mediator distributions rather than fixing mediators at individual level nested counterfactual values, this framework avoids the nested individual level cross-world counterfactuals required by natural direct and indirect effects and yields policy-relevant direct and indirect effect summaries.

Building on this formulation, subsequent work extended interventional mediation to increasingly complex settings, including multiple mediators with unknown causal structures~\cite{vansteelandt2017interventional, lin2017interventional,loh2022disentangling}, longitudinal settings with time-varying exposures and mediators~\cite{vanderweele2017mediation}, and nonlinear models with high-dimensional mediators~\cite{loh2022nonlinear}. Moreno-Betancur et al.~\cite{moreno2021mediation} further enriched the conceptual framework of interventional mediation by aligning it with target trial emulation, providing a policy oriented interpretation through hypothetical shifts in mediator distributions.

Along related lines, Benkeser and Ran~\cite{benkeser2021nonparametric} proposed nonparametric and semiparametric estimation and inference for interventional mediation effects. More recently, Zhou and Wodtke~\cite{zhou2025causal} developed a Monte Carlo based framework for mediation analysis with multiple mediators using flexible neural conditional distribution models. Their framework focuses on mean based mediation estimands, including ITEs, IDEs, IIEs, and path-specific effects under the natural effect framework. DCMA builds on this simulation based perspective, but reconstructs IPSEs under the interventional mediation framework and extends the target from mean effects to full interventional outcome distributions.

\subsection{Distributional Causal Inference}

Quantile and distributional treatment effects have been widely studied as alternatives to mean based causal summaries~\cite{imbens1997estimating,abadie2002bootstrap, athey2006identification,firpo2007efficient,rothe2010nonparametric, chernozhukov2013inference,byambadalai2025efficient}. Related work has also studied causal targets defined by functionals or discrepancies of counterfactual outcome distributions ~\cite{kennedy2023semiparametric,kim2018causal,kallus2023robust}. DCMA builds on this distributional perspective but shifts the focus to interventional mediation, reconstructing the interventional outcome distributions underlying ITEs, IDEs, IIEs, and IPSEs.

Kernel and embedding methods provide another way to represent distribution valued causal objects. Counterfactual mean embeddings and conditional mean embeddings have been used to represent counterfactual distributions and conditional distributional treatment effects ~\cite{muandet2021counterfactual,park2021conditional}. Singh et al.~\cite{singh2025sequential} use sequential kernel embeddings to estimate mediated and time-varying dose-response curves through nested \(g\)-formula functionals. DCMA instead uses conditional generators and Monte Carlo forward simulation for mediation-specific distribution reconstruction, whereas RKHS embeddings are used only as an analytic device for deriving the ED based mediator-stage and outcome stage error decomposition.

\section{Causal Setup and Distributional Estimands}
\label{sec:setup}

Let \(Y\) be the outcome, \(A\in\{0,1\}\) a binary treatment,
\(\bm M=(M_1,\ldots,M_S)\) a vector of \(S\) mediators, and \(\bm Z\) a vector
of baseline covariates. For treatment level $a$ and
mediator value $\bm m=(m_1,\dots,m_S)$, let $Y_{a\bm m}$ denote the
potential outcome if
\(A\) were set to \(a\) and \(\bm M\) were set to \(\bm m\), and let
\(\bm M_a\) denote the mediator vector if \(A\) were set to \(a\).

Interventional mediation estimands are defined by drawing mediator values from
treatment induced mediator distributions. We write
\[
\tilde{\bm M}_a \mid \bm Z=\bm z
\sim
P(\bm M_a\mid \bm Z=\bm z)
\]
for a random mediator vector drawn from the mediator distribution induced by
treatment level $a$ conditional on baseline covariates. Thus,
$Y_{a\tilde{\bm M}_{a'}}$ denotes the potential outcome under treatment level
$a$ when the mediator vector is drawn from the distribution induced by
treatment level $a'$.

For the purpose of defining IPSEs, it is convenient to partition the mediator
vector around the target mediator \(M_s\). We write
\[
\begin{aligned}
\tilde{\bm M}^{(<s)}_a
&=(\tilde M_{1a},\dots,\tilde M_{(s-1)a}),\\
\tilde{\bm M}^{(>s)}_a
&=(\tilde M_{(s+1)a},\dots,\tilde M_{Sa}).
\end{aligned}
\]
The ordering used here is an indexing convention for defining IPSEs, not an
assumption that the mediators follow a complete causal ordering.
Interventional mediation effects are defined through intervention distributions
assigned to the mediator vector, and therefore do not require specifying a
complete causal ordering among mediators~\cite{vansteelandt2017interventional}.
For graphical interpretation relative to the underlying mediator directed acyclic graph (DAG), the IPSE
for \(M_s\) can be viewed as the interventional contrast associated with paths
from \(A\) to \(Y\) whose last mediator before the outcome is \(M_s\). Figure~\ref{fig-Second} illustrates this idea in a two mediator DAG. 
The role of mediator indexing in IPSE
is further discussed in Appendix~\ref{app:ordering}.

\begin{figure}[htbp]
    \centering
    \resizebox{2.35in}{!}{%
        \begin{tikzpicture}[
            node distance=1cm and 1.5cm,
            font=\sffamily\scriptsize,
            observed/.style={rectangle, rounded corners, draw, thick, 
                              minimum width=0.8cm, minimum height=0.4cm, align=center},            
            >={Stealth[length=4pt]},
            every edge/.style={->, thin},
        ]
            \node[observed] (A) {$A$};
            \node[observed, above right=3mm and 5mm of A] (M1) {$M_1$};
            \node[observed, below right=3mm and 5mm of A] (M2) {$M_2$};
            \node[observed, below right=3mm and 6mm of M1] (Y) {$Y$};
            \node[observed, above left=5.5mm and 4.5mm of A] (Z) {$Z$};		

            \draw[->, thin, draw={rgb,255:red,31; green,119; blue,180}] (A) -- (M1);
            \draw[->, thin, draw={rgb,255:red,255; green,127; blue,14}] (A) -- (M2);
            \draw[->, thin, draw={rgb,255:red,44; green,160; blue,44}] (A) -- (Y);
            \draw[->, thin, draw={rgb,255:red,31; green,119; blue,180}] (M1) -- (Y);
            \draw[->, thin, draw={rgb,255:red,255; green,127; blue,14}] (M1) -- (M2);
            \draw[->, thin, draw={rgb,255:red,255; green,127; blue,14}] (M2) -- (Y);

            \draw[->, thin, draw=black!50] (Z) -- (A);
            \draw[->, thin, draw=black!50] (Z) to[bend left=10] (M1);
            \draw[->, thin, draw=black!50] (Z) to[bend right=40] (M2);
            \draw[->, thin, draw=black!50] (Z) to[bend left=45] (Y);
        \end{tikzpicture}
    }
    \caption{\label{fig-Second}%
    Causal directed acyclic graph (DAG) with exposure $A$, two mediators $M_1$, $M_2$, outcome $Y$, and baseline covariates $\bm{Z}$.  
    The green edge denotes the IDE,  
    the blue edges denote the IPSE through $M_1$, 
    and the orange edges denote the IPSE through $M_2$.}
\end{figure}
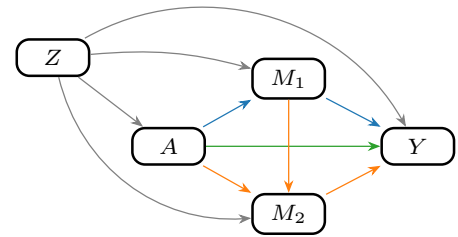

\subsection{Distributional Estimands}
\label{subsec:estimands}

Building on existing interventional mediation estimands~\cite{vanderweele2014effect,vansteelandt2017interventional}, we take the interventional outcome distributions as the primary targets and define
mediation estimands by applying user-specified functionals to these distributions.
Let \(\mathcal P(\mathbb R)\) denote the set of probability distributions on the outcome
space, and let
\[
\Psi:\mathcal P(\mathbb R)\times\mathcal P(\mathbb R)\to\mathbb R
\]
be a user-specified functional comparing two outcome distributions.

We define
the distributional ITE, IDE, IIE, and IPSE
through mediator $M_s$ as
\[
\begin{aligned}
\mathrm{ITE}^{\Psi}
&:=
\Psi\!\left(
P_{Y_{1\tilde{\bm M}_1}},
P_{Y_{0\tilde{\bm M}_0}}
\right),\\
\mathrm{IDE}^{\Psi}
&:=
\Psi\!\left(
P_{Y_{1\tilde{\bm M}_0}},
P_{Y_{0\tilde{\bm M}_0}}
\right),\\
\mathrm{IIE}^{\Psi}
&:=
\Psi\!\left(
P_{Y_{1\tilde{\bm M}_1}},
P_{Y_{1\tilde{\bm M}_0}}
\right),\\
\mathrm{IPSE}^{\Psi}_{s}
&:=
\Psi\!\left(
P_{Y_{1\tilde{\bm M}^{(<s)}_0 \tilde M_{s1}\tilde{\bm M}^{(>s)}_1}},
P_{Y_{1\tilde{\bm M}^{(<s)}_0 \tilde M_{s0}\tilde{\bm M}^{(>s)}_1}}
\right).
\end{aligned}
\]
Here \(P_{\cdot}\) denotes the corresponding interventional outcome
distribution. The functional \(\Psi\) can be chosen according to the
scientific target.

A first class of functionals is given by contrast type functionals,
\[
\Psi(P,Q)=T(P)-T(Q),
\]
where $T$ is a scalar summary of a distribution. If
$T(P)=\int y\,P(dy)$, the above definitions reduce to the usual mean based
interventional mediation effects. Different
choices of $T$ target different scalar features of the interventional outcome
distributions. For example, $T_c(P)=P(Y\ge c)$ gives an exceedance risk effect
at threshold $c$, $T_t(P)=F_P(t)$ gives a CDF-based effect at point $t$, and
\[
T_{\tau}(P)=F_P^{-1}(\tau)
=
\inf\{y:F_P(y)\ge \tau\},
\qquad 0<\tau<1,
\]
gives a quantile effect.

A second class is given by discrepancy based functionals,
\[
\Psi(P,Q)=D(P,Q),
\]
such as the ED~\cite{rizzo2016energy} or the Wasserstein
distance~\cite{villani2009optimal}. These discrepancy based summaries quantify global differences between
interventional outcome distributions, including changes in shape, spread, modality,
and tail behavior that may be missed by a single summary contrast.

\subsection{Identification}
\label{subsec:identification}

We extend the standard interventional mediation \(g\)-formula
~\cite{vanderweele2014effect,vansteelandt2017interventional} from mean effects
to interventional outcome distributions.
For a Borel set $B\subseteq\mathbb R$, define the observed-data outcome distribution
\[
K_a^z(B\mid \bm m)
:=
\mathbb P(Y\in B\mid A=a,\bm M=\bm m,\bm Z=\bm z),
\]
and the observed mediator distribution
\[
G_a^z(d\bm m)
:=
P_{\bm M\mid A=a,\bm Z=\bm z}(d\bm m).
\]

For mediator-specific effects, let
$G_a^{z,<s}$, $G_a^{z,s}$, and $G_a^{z,>s}$ denote the marginal distributions of
$\bm M^{(<s)}$, $M_s$, and $\bm M^{(>s)}$ under $G_a^z$, respectively. For
$r\in\{0,1\}$, define
\[
Q_{s,r}^z(d\bm m)
=
G_0^{z,<s}(d\bm m^{(<s)})
G_r^{z,s}(dm_s)
G_1^{z,>s}(d\bm m^{(>s)}),
\]
where $\bm m=(\bm m^{(<s)},m_s,\bm m^{(>s)})$. This distribution represents the
mediator distribution used in the IPSE definition through $M_s$.

Let
\[
\begin{aligned}
R_{a,a'}^z(B)
&:=
P(Y_{a\tilde{\bm M}_{a'}}\in B\mid \bm Z=\bm z),\\
R_{a,a'}(B)
&:=
\int R_{a,a'}^z(B)\,P_{\bm Z}(d\bm z)
\end{aligned}
\]
denote the conditional and marginal interventional outcome distributions for the
aggregate effects. For the IPSE through mediator \(M_s\), let
\[
\begin{aligned}
R_{s,r}^z(B)
&:=
P\!\left(
Y_{1,\tilde{\bm M}^{(<s)}_0,\tilde M_{sr},
\tilde{\bm M}^{(>s)}_1}\in B
\mid \bm Z=\bm z
\right),\\
R_{s,r}(B)
&:=
\int R_{s,r}^z(B)\,P_{\bm Z}(d\bm z)
\end{aligned}
\]
denote the corresponding conditional and marginal interventional outcome distributions.

We use the following standard assumptions to identify these interventional outcome distributions from the observed data.

\begin{assumption}[Consistency]
\label{ass:consistency}
For any $a\in\{0,1\}$ and mediator value $\bm m$, if
$(A,\bm M)=(a,\bm m)$, then $Y=Y_{a\bm m}$. Moreover, if $A=a$, then
$\bm M=\bm M_a$.
\end{assumption}

\begin{assumption}[Positivity]
\label{ass:positivity}
For \(P_{\bm Z}\)-almost every \(\bm z\) and each \(a\in\{0,1\}\),
\[
0<\mathbb P(A=a\mid \bm Z=\bm z)<1.
\]
The mediator values used to evaluate \(K_a^z(\cdot\mid\bm m)\) in the
identification formula must lie in the support of
\(P(\bm M\mid A=a,\bm Z=\bm z)\).
\end{assumption}

\begin{assumption}[Sequential ignorability]
\label{ass:si}
For all relevant $a\in\{0,1\}$ and mediator values $\bm m$,
\[
\begin{aligned}
Y_{a\bm m} &\perp\!\!\!\perp A \mid \bm Z,\\
\bm M_a &\perp\!\!\!\perp A \mid \bm Z,\\
Y_{a\bm m} &\perp\!\!\!\perp \bm M \mid A,\bm Z .
\end{aligned}
\]
\end{assumption}

Assumption~\ref{ass:si} rules out unmeasured confounding of the
treatment-outcome, treatment-mediator, and mediator-outcome relationships
after adjustment for baseline covariates. These are standard identifying
conditions in interventional mediation analysis~\cite{vanderweele2014effect,vansteelandt2017interventional}. Their plausibility depends
on the study design and the adequacy of the measured covariates.

\begin{theorem}[Identification of interventional outcome distributions]
\label{thm:id-main}
Under Assumptions~\ref{ass:consistency}--\ref{ass:si}, for any
\(a,a'\in\{0,1\}\),
\[
R_{a,a'}^z(B)
=
\int K_a^z(B\mid \bm m)\,G_{a'}^z(d\bm m).
\]
For the IPSE through mediator \(M_s\), the outcome treatment is fixed at
\(1\). For \(r\in\{0,1\}\),
\[
R_{s,r}^z(B)
=
\int K_1^z(B\mid \bm m)\,Q_{s,r}^z(d\bm m).
\]
Consequently, the corresponding marginal distributions \(R_{a,a'}\) and \(R_{s,r}\)
are identified by integrating the conditional distributions over \(P_{\bm Z}\). For any
\(\Psi\) that is well defined on these distributions,
\[
\begin{aligned}
\mathrm{ITE}^{\Psi}
&=
\Psi(R_{1,1},R_{0,0}),\\
\mathrm{IDE}^{\Psi}
&=
\Psi(R_{1,0},R_{0,0}),\\
\mathrm{IIE}^{\Psi}
&=
\Psi(R_{1,1},R_{1,0}),\\
\mathrm{IPSE}_{s}^{\Psi}
&=
\Psi(R_{s,1},R_{s,0}).
\end{aligned}
\]
\end{theorem}

The fully expanded IPSE formula and the proof of
Theorem~\ref{thm:id-main} are given in the
Supplementary Material. The identification formulas show that the target
interventional outcome distributions are determined by the mediator distribution
\(P(\bm M\mid A,\bm Z)\), the outcome distribution
\(P(Y\mid A,\bm M,\bm Z)\), and the covariate distribution. Estimation of
distributional mediation effects is therefore reduced to conditional
distribution learning, motivating the generative plug-in estimator developed
next.

\section{DCMA Framework}
\label{sec:generative}

DCMA has two components: conditional distribution learning and interventional
reconstruction. It first learns the mediator and outcome conditional distributions with
noise driven generators, and then samples from the fitted generators under the
target intervention regimes to reconstruct interventional outcome distributions by
Monte Carlo forward simulation. The corresponding plug-in estimates are obtained by applying \(\Psi\) to these reconstructed distributions.
Figure~\ref{fig:pipeline} gives an overview of the workflow.

\begin{figure}[t]
    \centering
    \includegraphics[width=0.85\linewidth]{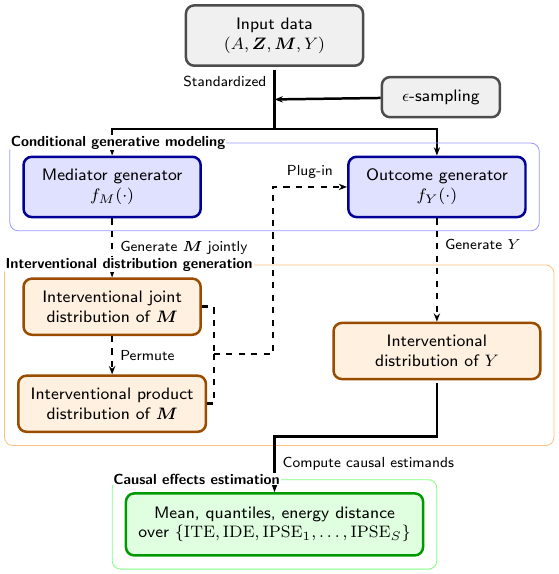}
    \caption{\label{fig:pipeline}
    Overview of the DCMA framework. Conditional generators learn the mediator and outcome conditional distributions, and Monte Carlo forward simulation reconstructs the target interventional outcome distributions.}
\end{figure}

\subsection{Conditional Generators}

A noise driven conditional generator represents a conditional distribution through a measurable map \(g(X,\varepsilon)\), where \(X\) denotes the conditioning variables and \(\varepsilon\) is an exogenous noise vector sampled independently of \(X\). In DCMA, this template is instantiated twice, once for the mediator conditional
distribution and once for the outcome conditional distribution. This conditional distribution learning strategy is closely related to deep
conditional sampling, Wasserstein generative regression (WGR), and engression
~\cite{zhou2023deep,song2026wasserstein,shen2025engression}.

Let \(\bm\varepsilon_M\sim P_{\bm\varepsilon_M}\) be an exogenous noise vector sampled independently of the generator inputs \((a,\bm z)\). Given \((a,\bm z)\), drawing \(\bm\varepsilon_M\) and evaluating
\(f_M(a,\bm z,\bm\varepsilon_M)\) induces a conditional distribution for the mediator vector,
\[
P_{f_M}(\cdot\mid a,\bm z)
=
\mathcal L\{f_M(a,\bm z,\bm\varepsilon_M)\},
\]
where \(\mathcal L(\cdot)\) denotes the distribution of a random variable. At the population optimum, the induced distribution matches the observed-data mediator conditional distribution,
\[
P_{f_M}(\cdot\mid a,\bm z)
=
P_{\bm M\mid A=a,\bm Z=\bm z},
\]
for almost every $(a,\bm z)$.
The mediator generator is fitted jointly for
\(\bm M=(M_1,\dots,M_S)\), so that generated draws can preserve conditional
dependence among mediators given treatment and covariates.

Similarly, with exogenous noise
\(\bm\varepsilon_Y\sim P_{\bm\varepsilon_Y}\) sampled independently of
\((a,\bm m,\bm z)\), the outcome generator induces a conditional distribution for the outcome,
\[
P_{f_Y}(\cdot\mid a,\bm m,\bm z)
=
\mathcal L\{f_Y(a,\bm m,\bm z,\bm\varepsilon_Y)\},
\]
with population optimum

\subsection{Training Objective}

Let \(\mathcal D(P,Q)\) denote a population level discrepancy between
probability distributions satisfying
\(\mathcal D(P,Q)\ge 0\) and \(\mathcal D(P,Q)=0\) if and only if \(P=Q\) on the relevant class of distributions.
The population mediator and outcome objectives are
\[
\begin{aligned}
\mathcal L_M(f_M)
&=
\mathbb E\!\left[
\mathcal D\{P_{f_M}(\cdot\mid A,\bm Z),
P_{\bm M\mid A,\bm Z}\}
\right],\\
\mathcal L_Y(f_Y)
&=
\mathbb E\!\left[
\mathcal D\{P_{f_Y}(\cdot\mid A,\bm M,\bm Z),
P_{Y\mid A,\bm M,\bm Z}\}
\right].
\end{aligned}
\]
If the generator classes \(\mathcal M\) and \(\mathcal Y\) are sufficiently rich and the population objectives are minimized exactly to zero, the population optimal generators satisfy
\[
\begin{aligned}
P_{f_M^\ast}(\cdot \mid a,\bm z)
&=
P_{\bm M \mid A=a,\bm Z=\bm z},\\
P_{f_Y^\ast}(\cdot \mid a,\bm m,\bm z)
&=
P_{Y \mid A=a,\bm M=\bm m,\bm Z=\bm z},
\end{aligned}
\]
for almost all relevant conditioning values. 

The main implementation uses an energy score (ES) based conditional generative loss.
For
$P\in\mathcal P_1(\mathbb R^d)$ and observation $x$, define
\[
S_{\rm ES}(P,x)
=
\mathbb E_{U\sim P}\|U-x\|
-
\frac12
\mathbb E_{U,U'\sim P}\|U-U'\|,
\]
where $U$ and $U'$ are independent draws from $P$. The ES is strictly
proper~\cite{gneiting2007strictly}. Specifically, for any true distribution
$Q\in\mathcal P_1(\mathbb R^d)$,
\[
\mathbb E_{X\sim Q}S_{\rm ES}(Q,X)
\le
\mathbb E_{X\sim Q}S_{\rm ES}(P,X),
\]
with equality if and only if $P=Q$.

Given $L\ge 2$ independent noise draws
$\bm\varepsilon_{i1}^{M},\ldots,\bm\varepsilon_{iL}^{M}
\overset{\mathrm{i.i.d.}}{\sim} P_{\bm\varepsilon_M}$ for each observation
$i$, the Monte Carlo approximation for the mediator stage is
\[
\widehat S_{{\rm ES},i}^{M}
=
\frac1L\sum_{\ell=1}^L
\|U_{i\ell}-\bm M_i\|
-
\frac{1}{2L(L-1)}
\sum_{\substack{\ell,\ell'=1\\ \ell\neq \ell'}}^L
\|U_{i\ell}-U_{i\ell'}\|,
\]
where
\[
U_{i\ell}
=
f_M(A_i,\bm Z_i,\bm\varepsilon_{i\ell}^{M}).
\]
The empirical mediator loss is
\[
\widehat{\mathcal L}_M(f_M)
=
\frac1n\sum_{i=1}^n
\widehat S_{{\rm ES},i}^{M}.
\]

The outcome loss is defined analogously. Given
\[
V_{i\ell}
=
f_Y(A_i,\bm M_i,\bm Z_i,\bm\varepsilon_{i\ell}^{Y}),
\qquad
\bm\varepsilon_{i\ell}^{Y}
\overset{\mathrm{i.i.d.}}{\sim} P_{\bm\varepsilon_Y},
\]
for \(\ell=1,\ldots,L\), define
\[
\widehat S_{{\rm ES},i}^{Y}
=
\frac1L\sum_{\ell=1}^L
\|V_{i\ell}-Y_i\|
-
\frac{1}{2L(L-1)}
\sum_{\substack{\ell,\ell'=1\\ \ell\neq \ell'}}^L
\|V_{i\ell}-V_{i\ell'}\|.
\]
The empirical outcome loss is
\[
\widehat{\mathcal L}_Y(f_Y)
=
\frac1n\sum_{i=1}^n
\widehat S_{{\rm ES},i}^{Y}.
\]

Other conditional generative losses can be used without changing the
interventional reconstruction step. As an example, we implement a WGR-based version of DCMA~\cite{song2026wasserstein}, which replaces the ES loss while keeping the same reconstruction procedure.

\subsection{Monte Carlo Reconstruction}

After training, let $\widehat f_M$ and $\widehat f_Y$ denote the fitted
generators. For each observed covariate vector $\bm Z_i$ and Monte Carlo draw
$b=1,\dots,B$, generate
\[
\widehat{\bm M}_{a}^{(i,b)}
=
\widehat f_M(a,\bm Z_i,\bm\varepsilon_{M,a}^{(i,b)}),
\qquad a\in\{0,1\},
\]
with independently sampled noise. Then generate
\[
\widehat Y_{aa'}^{(i,b)}
=
\widehat f_Y
\left(
a,
\widehat{\bm M}_{a'}^{(i,b)},
\bm Z_i,
\bm\varepsilon_{Y,aa'}^{(i,b)}
\right),
\qquad a,a'\in\{0,1\}.
\]
For each pair $(a,a')$, define the empirical interventional outcome distribution
\[
\widehat P_{aa'}
=
\frac{1}{nB}
\sum_{i=1}^n\sum_{b=1}^B
\delta_{\widehat Y_{aa'}^{(i,b)}},
\]
where $\delta_y$ denotes the point mass at $y$. The plug-in estimators are
\[
\begin{aligned}
\widehat{\mathrm{ITE}}^\Psi
&=
\Psi(\widehat P_{11},\widehat P_{00}),\\
\widehat{\mathrm{IDE}}^\Psi
&=
\Psi(\widehat P_{10},\widehat P_{00}),\\
\widehat{\mathrm{IIE}}^\Psi
&=
\Psi(\widehat P_{11},\widehat P_{10}).
\end{aligned}
\]

For IPSE reconstruction, DCMA constructs mediator inputs according to the IPSE identification formula by combining mediator blocks generated under different treatment levels. For each fixed observed covariate vector \(\bm Z_i\), this product type construction is approximated by randomly permuting precomputed mediator draws over the Monte Carlo index.
For mediator $M_s$, let $\pi_{<}^{(i,s)}$ and
$\pi_{>}^{(i,s)}$ be independent random permutations of $\{1,\ldots,B\}$.
For $r\in\{0,1\}$, define
\[
\begin{aligned}
\widehat{\bm H}_{s,r}^{(i,b)}
=
\big(&
\widehat{\bm M}_{0}^{(i,\pi_{<}^{(i,s)}(b)),(<s)},
\widehat M_{r}^{(i,b),s},
\widehat{\bm M}_{1}^{(i,\pi_{>}^{(i,s)}(b)),(>s)}
\big).
\end{aligned}
\]
The corresponding outcome draw is
\[
\widehat Y_{s,r}^{(i,b)}
=
\widehat f_Y
\left(
1,
\widehat{\bm H}_{s,r}^{(i,b)},
\bm Z_i,
\bm\varepsilon_{Y,s,r}^{(i,b)}
\right).
\]
The empirical interventional outcome distribution is
\[
\widehat P_{s,r}
=
\frac{1}{nB}
\sum_{i=1}^n\sum_{b=1}^B
\delta_{\widehat Y_{s,r}^{(i,b)}},
\]
where \(\delta_y\) denotes the Dirac point mass at \(y\). The plug-in estimator
of the distributional IPSE for \(M_s\) is
\[
\widehat{\mathrm{IPSE}}_{s}^{\Psi}
=
\Psi(\widehat P_{s,1},\widehat P_{s,0}).
\]

\begin{algorithm}[t]
\caption{DCMA plug-in estimator}
\label{alg:dcma}
\begin{algorithmic}[1]
\Require Data $\{(Y_i,A_i,\bm M_i,\bm Z_i)\}_{i=1}^n$,
Monte Carlo size $B$, functional $\Psi$
\State Train $\widehat f_M$ by minimizing $\widehat{\mathcal L}_M$.
\State Train $\widehat f_Y$ by minimizing $\widehat{\mathcal L}_Y$.

\For{$i=1,\ldots,n$}
  \For{$b=1,\ldots,B$}
    \State Generate $\widehat{\bm M}_{0}^{(i,b)}$ and
    $\widehat{\bm M}_{1}^{(i,b)}$ from $\widehat f_M$.
    \State Generate $\widehat Y_{00}^{(i,b)}$,
    $\widehat Y_{10}^{(i,b)}$, and $\widehat Y_{11}^{(i,b)}$
    from $\widehat f_Y$.
  \EndFor

  \For{$s=1,\ldots,S$}
    \State Draw independent random permutations
    $\pi_{<}^{(i,s)}$ and $\pi_{>}^{(i,s)}$ of $\{1,\ldots,B\}$.
    \For{$b=1,\ldots,B$}
      \For{$r\in\{0,1\}$}
        \State Construct $\widehat{\bm H}_{s,r}^{(i,b)}$
        by combining the reordered mediator blocks.
        \State Generate $\widehat Y_{s,r}^{(i,b)}$ from $\widehat f_Y$
        evaluated at $(1,\widehat{\bm H}_{s,r}^{(i,b)},\bm Z_i,\bm\varepsilon_{Y,s,r}^{(i,b)})$.
      \EndFor
    \EndFor
  \EndFor
\EndFor

\State Form empirical outcome distributions $\widehat P_{aa'}$ and $\widehat P_{s,r}$.
\State Return
$\widehat{\mathrm{ITE}}^\Psi$,
$\widehat{\mathrm{IDE}}^\Psi$,
$\widehat{\mathrm{IIE}}^\Psi$, and
$\{\widehat{\mathrm{IPSE}}_{s}^{\Psi}\}_{s=1}^S$.
\end{algorithmic}
\end{algorithm}

\section{Error Analysis}
\label{sec:error}

This section studies how learning errors in the mediator and outcome conditional distributions propagate to reconstructed interventional outcome distributions. The analysis is based on the ED/RKHS representation, which is aligned with the ES-based training objective.

\subsection{Energy Distance and RKHS Representation}
\label{subsec:ed-es}

We first recall the ED and its RKHS representation. Let
\(\mathcal P_1(\mathbb R^d)\) denote the set of probability distributions with finite
first moments. For \(P,Q\in\mathcal P_1(\mathbb R^d)\), the ED is
\[
\begin{aligned}
\mathrm{ED}(P,Q)
=
2\mathbb E\|X-Y\|
-\mathbb E\|X-X'\|-\mathbb E\|Y-Y'\|,
\end{aligned}
\]
where $X,X'\sim P$ and $Y,Y'\sim Q$ are independent. The ED is
nonnegative and equals zero if and only if $P=Q$.

For the loss oriented ES defined in Section~\ref{sec:generative},
the excess risk satisfies
\[
\mathbb E_{Y\sim Q}S_{\rm ES}(P,Y)
-
\mathbb E_{Y\sim Q}S_{\rm ES}(Q,Y)
=
\frac12\,\mathrm{ED}(P,Q).
\]
Thus, ES-based conditional distribution learning is naturally aligned with
ED control of the learned conditional distributions.

The ED also admits a kernel mean embedding representation through the
equivalence between distance based and RKHS based statistics for semimetrics
of negative type~\cite{sejdinovic2013equivalence}. Specifically, let $k_Y$ be a positive definite kernel on the outcome space associated with the
negative type semimetric underlying the ED, and let $\mathcal H_Y$ be its
RKHS. Then
\[
\mathrm{ED}_Y(P,Q)
=
2\|\mu_Y(P)-\mu_Y(Q)\|_{\mathcal H_Y}^2,
\]
where
\[
\mu_Y(P)=\int k_Y(\cdot,y)\,P(dy).
\]
We write
\[
\mathsf d_Y(P,Q)
:=
\left\{\frac12\mathrm{ED}_Y(P,Q)\right\}^{1/2}
=
\|\mu_Y(P)-\mu_Y(Q)\|_{\mathcal H_Y}.
\]
Analogously, for mediator distributions, let $k_M$ be the corresponding kernel on the
mediator space and define
\[
\mathsf d_M(P,Q)
:=
\left\{\frac12\mathrm{ED}_M(P,Q)\right\}^{1/2}.
\]

\subsection{Conditional Interventional distribution Error}
\label{subsec:structural-decomp}

Fix a treatment level \(a\) and a covariate value \(\bm z\). Let \(H^z\) denote the mediator intervention distribution to be combined with the conditional outcome distribution \(K_a^z(\cdot\mid\bm m)\). For example, \(H^z=G_{a'}^z\) for ITE, IDE, or IIE, whereas \(H^z=Q_{s,r}^z\) for IPSE. Let \(\widehat H^z\) denote the corresponding distribution induced by the fitted mediator generator. Define the fitted conditional outcome distribution induced by the learned outcome generator as
\[
\widehat K_a^z(\cdot\mid \bm m)
=
\mathcal L\{\widehat f_Y(a,\bm m,\bm z,\bm\varepsilon_Y)\}.
\]
The true and estimated conditional interventional outcome distributions are
\[
\begin{aligned}
R^z
&=
\int K_a^z(\cdot\mid \bm m)\,H^z(d\bm m),\\
\widehat R^z
&=
\int \widehat K_a^z(\cdot\mid \bm m)\,\widehat H^z(d\bm m).
\end{aligned}
\]
Define the fitted conditional outcome embedding
\[
\widehat\phi_{a,z}(\bm m)
=
\mu_Y\{\widehat K_a^z(\cdot\mid \bm m)\}
\in \mathcal H_Y.
\]

We impose the following regularity conditions to state the structural error
decomposition.

\textbf{Condition E.} 
\begin{enumerate}
\item[(i)]
For all relevant mediator values $\bm m$,
\[
K_a^z(\cdot\mid \bm m),\,
\widehat K_a^z(\cdot\mid \bm m)
\in
\mathcal P_1(\mathbb R).
\]
Their kernel mean embeddings are measurable and integrable under the mediator
distributions used in the decomposition.

\item[(ii)]
The mediator intervention distributions satisfy
\[
H^z,\widehat H^z\in \mathcal P_1(\mathbb R^S),
\]
so that $\mathsf d_M(\widehat H^z,H^z)$ is well defined.

\item[(iii)]
The learned outcome embedding map satisfies
\[
\widehat\phi_{a,z}\in \mathcal H_M\otimes\mathcal H_Y,
\qquad
\|\widehat\phi_{a,z}\|_{\mathcal H_M\otimes\mathcal H_Y}
\le L_\phi
\]
for some $L_\phi<\infty$.
\end{enumerate}

\begin{theorem}[Structural error bound]
\label{thm:ed-structural}
Under Condition E,
\[
\begin{aligned}
\mathsf d_Y(\widehat R^z,R^z)
&\le{}
L_\phi\,\mathsf d_M(\widehat H^z,H^z)\\
+&
\left[
\int
\mathsf d_Y^2\!\left(
\widehat K_a^z(\cdot\mid \bm m),
K_a^z(\cdot\mid \bm m)
\right)
H^z(d\bm m)
\right]^{1/2}.
\end{aligned}
\]
Consequently,
\[
\begin{aligned}
\mathrm{ED}_Y(\widehat R^z,R^z)
&\le{}
2L_\phi^2\,
\mathrm{ED}_M(\widehat H^z,H^z)\\
+&
2\int
\mathrm{ED}_Y\!\left(
\widehat K_a^z(\cdot\mid \bm m),
K_a^z(\cdot\mid \bm m)
\right)
H^z(d\bm m).
\end{aligned}
\]
\end{theorem}

\begin{proof}
    The proof is given in the Supplementary Material.
\end{proof}

Theorem~\ref{thm:ed-structural} separates reconstruction error into two
sources. The first term is the mediator stage error, measuring how errors in the fitted mediator intervention distribution propagate through the fitted outcome mechanism. The second term is the outcome stage error, measuring the conditional outcome distribution error averaged under the target mediator intervention distribution.

\begin{corollary}[Marginal interventional distribution error]
\label{cor:marginal-distribution-error}
Let
\[
R
=
\int R^z\,P_{\bm Z}(d\bm z),
\qquad
\widehat R
=
\int \widehat R^z\,P_{\bm Z}(d\bm z)
\]
be the corresponding marginal interventional outcome distributions. Under the conditions
of Theorem~\ref{thm:ed-structural},
\[
\mathsf d_Y(\widehat R,R)
\le
\int
\mathsf d_Y(\widehat R^z,R^z)\,
P_{\bm Z}(d\bm z).
\]
Consequently, the mediator and outcome stage decomposition in Theorem~\ref{thm:ed-structural} carries over to marginal interventional outcome distributions after averaging over baseline covariates.
\end{corollary}

Plug-in error bounds for contrast type summaries and ED based discrepancy summaries are given in the Supplementary Material.

\section{Experiments}
\label{sec:experiments}

\subsection{Synthetic Experiments} \label{subsec:synthetic} 

\paragraph{S1: Bimodal outcome}

We consider a single mediator synthetic setting with a binary treatment $A \sim \mathrm{Bernoulli}(0.5)$, a baseline covariate $Z \sim N(0,1)$, and sample size $n=5000$. The mediator follows
\[
M = 0.5 + A + 0.3Z + \varepsilon_M,
\qquad
\varepsilon_M \sim N(0,0.5^2).
\]
The outcome model is designed to produce a bimodal marginal interventional outcome distribution under treatment. For untreated individuals,
\[
Y = 4.3 + 0.5M + 0.2Z + \varepsilon_{Y0}.
\]
For treated individuals, define $S=\mathbbm{1}(Z\le 0)$ and let
\[
Y =
\begin{cases}
2.3 + 0.5M + 0.2Z + \varepsilon_{Y1}, & S=0,\\
6.3 + 0.5M + 0.2Z + \varepsilon_{Y2}, & S=1,
\end{cases}
\]
where \[ \varepsilon_{Y0},\varepsilon_{Y1},\varepsilon_{Y2} \overset{\mathrm{i.i.d.}}{\sim}N(0,1), \] and all mediator and outcome error terms are mutually independent.

Oracle interventional outcome distributions are computed by Monte Carlo intervention under the known data generating mechanism, and estimation accuracy is summarized by RMSE over 100 replications.
We compare DCMA--ES with three alternatives. DCMA--WGR uses the same interventional reconstruction framework but learns conditional distributions using WGR~\cite{song2026wasserstein}. Linear Gaussian fits linear conditional models with Gaussian residual noise~\cite{vansteelandt2017interventional} and reconstructs interventional outcome distributions by parametric Gaussian sampling.
MedFlow is implemented as a flow-based conditional simulation baseline following Zhou and Wodtke~\cite{zhou2025causal}.

\begin{figure*}[t]
\centering
\includegraphics[width=0.95\textwidth]{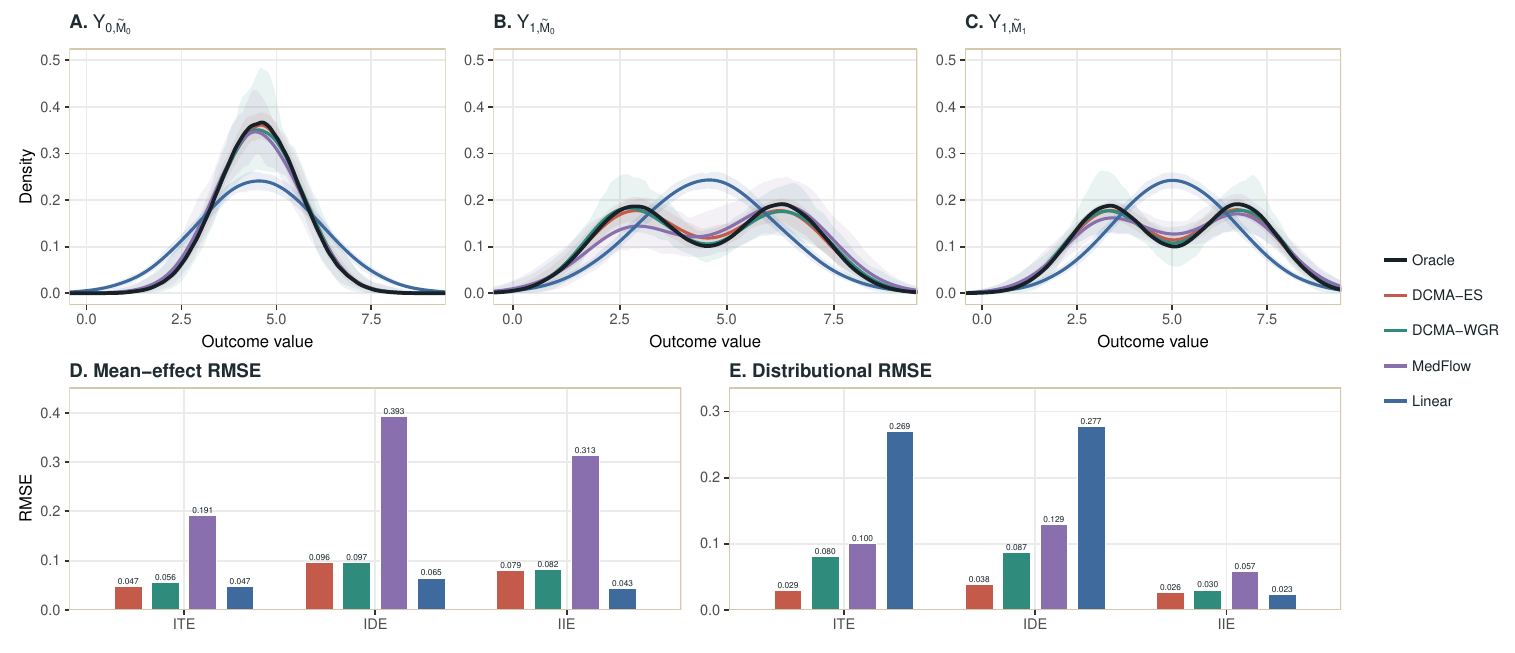}
\caption{
Synthetic bimodal experiment. Panels A--C compare oracle and method-specific estimated interventional outcome densities for \(Y_{0\tilde M_0}\), \(Y_{1\tilde M_0}\), and \(Y_{1\tilde M_1}\). Solid black curves denote oracle densities, colored curves denote estimated densities averaged across 100 replications, and shaded bands denote pointwise 95\% inter-replication intervals. Panels D--E compare DCMA--ES, DCMA--WGR, MedFlow, and Linear Gaussian by RMSE for mean based and ED based distributional effects over ITE, IDE, and
IIE. Lower values indicate better recovery.
}
\label{fig:s1_density}
\end{figure*}

Figure~\ref{fig:s1_density} summarizes the results. Panels A--C compare the oracle and estimated interventional densities across all methods. DCMA--ES and DCMA--WGR closely track the oracle densities and preserve the bimodal structure, whereas the Linear--Gaussian baseline is largely unimodal. Panels D--E report the corresponding RMSEs. DCMA--ES achieves the lowest ED RMSEs across ITE, IDE, and IIE, with DCMA--WGR giving similar distributional accuracy. For mean based effects, DCMA remains close to the Linear--Gaussian baseline and substantially outperforms MedFlow. Overall, this experiment shows that DCMA improves distributional recovery while retaining competitive accuracy for mean based mediation effects.

\paragraph{S2: Multiple dependent mediators} 

We next consider a multivariate mediation setting with \(S=5\) dependent
mediators. 

The mediators are generated as
\[
\bm{M}
=
0.5\,\bm{1}_5
+
\bm{b}_A A
+
\bm{b}_Z Z
+
\bm{\varepsilon}_M,
\qquad
\bm{\varepsilon}_M \sim \mathcal N(\bm{0}_5,\Sigma),
\]
where $\bm{1}_5$ and $\bm{0}_5$ denote the 5-dimensional vectors of ones and zeros,
and $\Sigma \in \mathbb{R}^{5\times 5}$ has entries
$\Sigma_{ij}=0.6^{|i-j|}$.
We set $\bm{b}_A=(1.0,\,0.8,\,0.6,\,0.4,\,0.2)$ and $\bm{b}_Z=(0.3,\,0.3,\,0.2,\,0.2,\,0.1)$.

The outcome is generated as
\[
Y
=
1+0.6A+0.2\,\bm 1_5^\top \bm M
+\sin(M_1M_2)+0.2Z+\varepsilon_Y,
\]
where \(\varepsilon_Y\sim N(0,1)\).  
All other settings are the same as in S1.
Table~\ref{tab:s2_dependent_mediators} shows that DCMA achieves low RMSEs for
mean based and ED based mediation effects in the presence of dependent
mediators and nonlinear mediator interactions in the outcome model.

\begin{table}[t]
\centering
\caption{Multiple dependent mediators experiment. Entries are RMSEs for mean based and ED based interventional mediation effect estimates over 100 replications.}
\label{tab:s2_dependent_mediators}
\small
\setlength{\tabcolsep}{4pt}
\renewcommand{\arraystretch}{1.08}
\resizebox{\columnwidth}{!}{%
\begin{tabular}{lccccccc}
\toprule
Metric
& IDE
& IIE
& IPSE$_1$
& IPSE$_2$
& IPSE$_3$
& IPSE$_4$
& IPSE$_5$ \\
\midrule
Mean RMSE
& 0.038 & 0.020 & 0.023 & 0.015 & 0.010 & 0.007 & 0.005 \\
ED RMSE
& 0.013 & 0.005 & 0.003 & 0.002 & 0.001 & 0.001 & 0.000 \\
\bottomrule
\end{tabular}%
}
\end{table}

\subsection{Semi-Synthetic IHDP Multi-Mediator Experiment}
\label{subsec:ihdp}

The third experiment evaluates whether DCMA can recover mediator-specific distributional patterns in a semi-synthetic multi-mediator setting. We use the IHDP dataset~\cite{hill2011bayesian,louizos2017causal} as the empirical base, retaining its binary treatment \(A\) and baseline covariates \(\bm Z\) while simulating three mediators and a continuous outcome from a known data generating mechanism. The mechanism is designed so that the three mediators affect distinct features of the conditional outcome distribution:
\[
\begin{aligned}
M_1 &\longrightarrow \mu
&&\text{(location channel)},\\
M_2 &\longrightarrow \log\sigma
&&\text{(scale channel)},\\
M_3 &\longrightarrow R
&&\text{(right-tail channel)}.
\end{aligned}
\]
Here, \(\mu\) denotes the conditional location, \(\sigma\) denotes the conditional scale, and \(R\) denotes the residual component that controls right-tail through a covariate dependent tail term. 
Full data-generating details are given in Appendix~\ref{app:ihdp-dgm}.

\begin{figure}[h]
\centering
\includegraphics[width=0.9\linewidth]{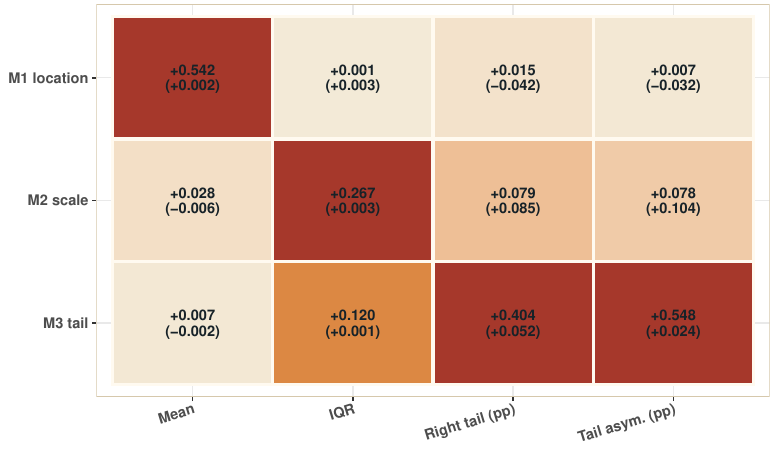}
\caption{
Semi-synthetic IHDP mediator-specific distributional patterns. Each cell reports the oracle value, with the DCMA estimation error in parentheses. The four summaries are mean shift, interquartile-range (IQR) change, right-tail probability change, and tail-asymmetry change. For each interventional outcome distribution, let \(Z_Y\) denote the standardized outcome. Right-tail probability is defined as \(P(Z_Y>2.5)\), and tail asymmetry is defined as \(P(Z_Y>2.5)-P(Z_Y<-2.5)\). Probability-based summaries are
reported in percentage points (pp).
}
\label{fig:ihdp_signature_heatmap}
\end{figure}

Figure~\ref{fig:ihdp_signature_heatmap} shows that DCMA closely tracks the oracle values and recovers the expected mediator-specific patterns. The \(M_1\)-specific intervention yields the largest mean shift, consistent with a location channel, whereas the \(M_2\)-specific intervention is most pronounced in the IQR summary and the \(M_3\)-specific intervention in the right-tail and tail-asymmetry summaries. 

\subsection{NHANES Liver Elastography Study}
\label{subsec:nhanes}

We apply DCMA to the 2017--2018 National Health and Nutrition Examination
Survey (NHANES) liver elastography data to examine how obesity, defined as BMI \(\ge 30\,\mathrm{kg/m^2}\), is associated with liver stiffness and whether this association is mediated by
metabolic pathways. The mediators are log HOMA-IR (homeostasis model assessment of insulin resistance) and the log TG/HDL-C ratio, used as markers of insulin resistance and dyslipidemia, respectively. Baseline covariates include age, sex, race, educational attainment, family income-to-poverty ratio, smoking status, alcohol use status, binge drinking, and physical activity. The final complete case analytic sample contains \(n=1,452\) participants, including 863 non-obese and 589 obese participants.

We summarize the reconstructed interventional outcome distributions by the mean difference in liver stiffness, the risk difference for exceeding \(8\,\mathrm{kPa}\)~\cite{berzigotti2021easl}, and ED. Uncertainty intervals are computed by bootstrap.

\begin{figure}[h]
\centering
\includegraphics[width=0.95\linewidth]{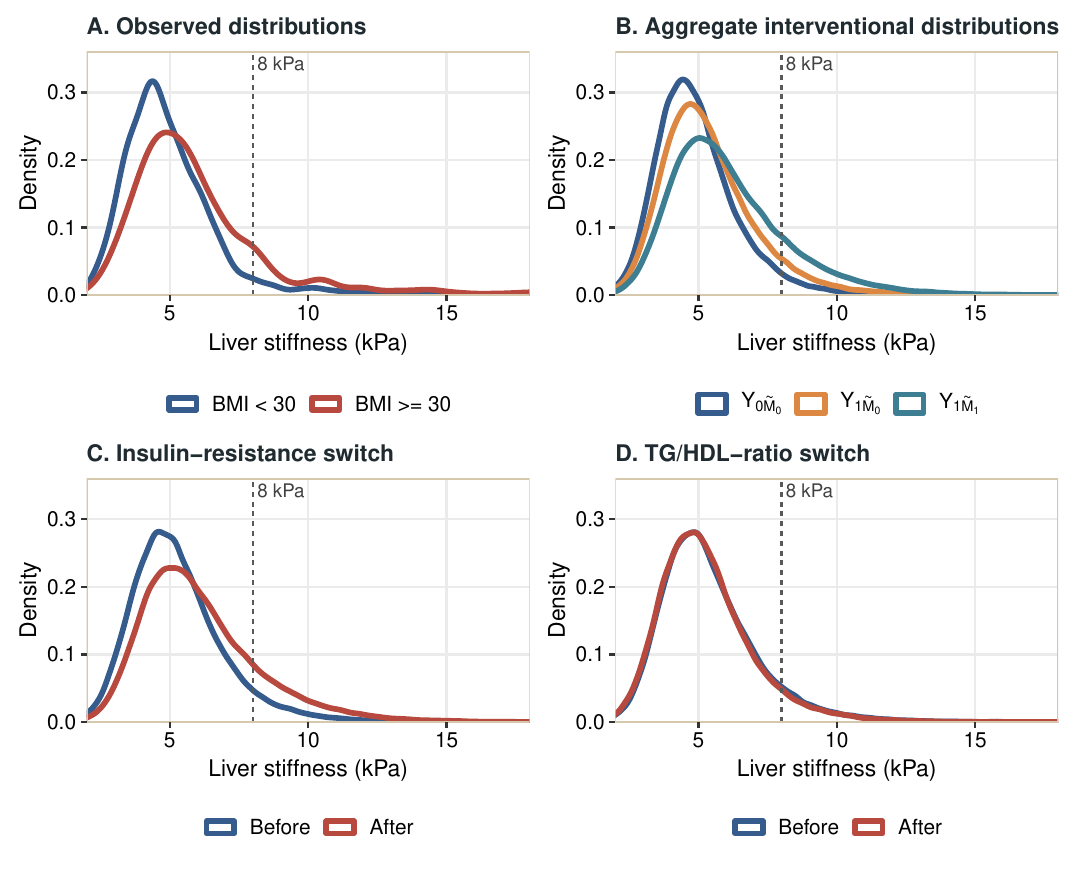}
\caption{
NHANES liver elastography study. Panel A shows observed liver stiffness distributions by obesity status, and Panel B shows reconstructed interventional distributions. Panels C--D show mediator-specific interventional reconstructions for the insulin-resistance and TG/HDL-ratio pathways. The vertical line marks the
\(8\) kPa threshold.
}
\label{fig:nhanes_density}
\end{figure}

Figure~\ref{fig:nhanes_density} shows that the obese group has a right-shifted liver stiffness distribution with a heavier upper tail. For the obesity comparison, the estimated ITE summaries are \(1.225\) kPa [0.790, 1.740] for mean liver stiffness, 12.1 percentage points [6.5, 18.0] for the exceedance risk difference, and 0.411 [0.176, 0.692] for ED (Table~\ref{tab:nhanes_effects}). The mediator-specific reconstructions suggest that the distributional shift is primarily aligned with the insulin-resistance pathway, with positive HOMA-IR path-specific estimates for both mean liver stiffness and exceedance risk. In contrast, the log TG/HDL-C pathway shows little distributional change, with mean and exceedance risk intervals covering zero.

\begin{table}[h]
\centering
\caption{
NHANES DCMA estimates. Entries are point estimates with bootstrap intervals.
Mean effects are in kPa, exceedance risk effects are risk differences for
$P(Y\ge 8\,\mathrm{kPa})$, and ED denotes energy distance.
}
\label{tab:nhanes_effects}
\small
\setlength{\tabcolsep}{4pt}
\renewcommand{\arraystretch}{1.08}
\resizebox{\columnwidth}{!}{%
\begin{tabular}{lccc}
\toprule
Effect
& Mean effect
& Exceedance risk
& ED \\
\midrule
ITE
& \shortstack{1.225\\{}[0.790, 1.740]}
& \shortstack{0.121\\{}[0.065, 0.180]}
& \shortstack{0.411\\{}[0.176, 0.692]} \\
IDE
& \shortstack{0.401\\{}[0.069, 0.722]}
& \shortstack{0.031\\{}[0.007, 0.068]}
& \shortstack{0.066\\{}[0.005, 0.160]} \\
IIE
& \shortstack{0.825\\{}[0.483, 1.348]}
& \shortstack{0.091\\{}[0.050, 0.130]}
& \shortstack{0.178\\{}[0.070, 0.347]} \\
HOMA-IR path
& \shortstack{0.856\\{}[0.558, 1.450]}
& \shortstack{0.091\\{}[0.052, 0.140]}
& \shortstack{0.193\\{}[0.090, 0.382]} \\
TG/HDL path
& \shortstack{$-0.030$\\{}[$-0.194$, 0.104]}
& \shortstack{$-0.001$\\{}[$-0.015$, 0.009]}
& \shortstack{0.002\\{}[0.000, 0.013]} \\
\bottomrule
\end{tabular}%
}
\end{table}

\subsection{Ablation Study}
\label{sec:ablation}

\paragraph{Outcome generator specification}
We first assess sensitivity to the outcome generator specification. With the
mediator generator and interventional reconstruction procedure fixed, we compare
three outcome generators:
\[
\begin{aligned}
Y
&=g_\theta(A,\bm M,\bm Z,\bm\varepsilon_Y),
&& \bm\varepsilon_Y\sim N(0,I),\\
Y
&=\mu_\theta(A,\bm M,\bm Z)
+\sigma_\theta(A,\bm M,\bm Z)\varepsilon,
&& \varepsilon\sim N(0,1),\\
Y
&=\mu_\theta(A,\bm M,\bm Z)+\sigma\varepsilon,
&& \varepsilon\sim N(0,1).
\end{aligned}
\]
They correspond to flexible noise injection, Gaussian location-scale noise, and homoskedastic Gaussian noise, respectively. As shown in Table~\ref{tab:ablation_noise}, the flexible generator yields the lowest ED RMSEs across IDE, IIE, and ITE, indicating improved recovery of the full interventional outcome distributions.

\begin{table}[h]
\centering
\caption{
Outcome-noise ablation in the synthetic bimodal setting. Entries are ED RMSEs over 100 replications.
}
\label{tab:ablation_noise}
\small
\setlength{\tabcolsep}{4pt}
\begin{tabular}{lccc}
\toprule
Outcome generator & IDE & IIE & ITE \\
\midrule
Flexible noise-injection & \textbf{0.038} & \textbf{0.026} & \textbf{0.029} \\
Gaussian location-scale  & 0.063 & 0.030 & 0.044 \\
Homoskedastic Gaussian   & 0.209 & 0.030 & 0.213 \\
\bottomrule
\end{tabular}%
\end{table}

\paragraph{Joint mediator modeling}
We next compare the proposed joint mediator generator with a variant that models the mediators separately in the synthetic multiple dependent mediators setting. The separate variant learns each marginal mediator distribution independently and combines the generated mediators during interventional reconstruction. 
As shown in Table~\ref{tab:joint-separate-ed-rmse}, the two specifications give similar ED RMSEs for the IPSEs, but the separate variant yields much larger ED RMSEs for the IIE and the ITE. 

\begin{table}[h] 
\centering 
\caption{Joint mediator modeling ablation in the synthetic multiple dependent mediators setting. Entries are ED RMSEs over 100 replications.} 
\label{tab:joint-separate-ed-rmse} 
\small 
\setlength{\tabcolsep}{4pt} 
\renewcommand{\arraystretch}{1.08} 
\resizebox{\columnwidth}{!}{%
\begin{tabular}{lcccccccc} 
\toprule 
Variant & ITE & IDE & IIE & IPSE1 & IPSE2 & IPSE3 & IPSE4 & IPSE5 \\ 
\midrule 
Joint & \textbf{0.020} & \textbf{0.013} & \textbf{0.005} & \textbf{0.003} & \textbf{0.002} & \textbf{0.001} & \textbf{0.001} & 0.000 \\ 
Separate & 0.043 & 0.015 & 0.014 & 0.004 & 0.002 & 0.001 & 0.001 & \textbf{0.000} \\ 
\bottomrule 
\end{tabular}%
} \end{table}

\section{Conclusion}

This paper introduced DCMA, a conditional generative framework for distributional causal mediation analysis with multiple mediators. DCMA extends interventional mediation analysis from mean level contrasts to interventional outcome distributions. By learning the mediator and outcome conditional distributions with noise driven generators and reconstructing interventional outcome distributions through
Monte Carlo forward simulation, the framework evaluates total, direct, indirect, and ordered path-specific interventional estimands through user-specified distributional functionals.

DCMA is useful when the scientific question concerns not only whether a pathway changes the outcome on average, but also how it changes the outcome distribution. Such questions arise in clinical and public health studies with tail risk or threshold based outcomes, environmental and policy studies with heterogeneous or tail sensitive responses, and omics studies where multiple
biological pathways may affect different features of the outcome distribution. In these settings, distributional mediation analysis provides a richer description of pathway-specific effects than a single summary contrast.

Several limitations should be noted. Like other mediation methods, DCMA relies on no unmeasured confounding for the treatment--mediator, treatment--outcome, and mediator--outcome relationships.
The credibility of this assumption depends on study design and covariate measurement, especially in observational applications. In addition, the current implementation focuses on cross-sectional mediators and outcomes. Extensions to longitudinal mediators and survival outcomes remain important directions for future work.

\appendices

\section{Role of Mediator Ordering in IPSE}
\label{app:ordering}

The IPSEs in the main text use an index based partition around the target mediator. This partition is a bookkeeping device for defining ordered IPSEs and does not require the mediators to follow a complete causal ordering. However, different mediator indexings may define different mediator-specific estimands. When a scientifically meaningful mediator ordering is available, the ordering can be chosen to reflect that structure. When no ordering is clearly preferred, one may report IPSEs across a small set of plausible orderings or use the order-averaged summary defined below.

For a pre-specified set \(\Pi\) of mediator orderings, define
\[
\overline{\mathrm{IPSE}}_{j}^{\Psi}
=
\frac{1}{|\Pi|}
\sum_{\pi\in\Pi}
\mathrm{IPSE}_{\pi,j}^{\Psi},
\]
where \(\pi\) denotes a permutation of the mediator indices and
\(\mathrm{IPSE}_{\pi,j}^{\Psi}\) is the ordered IPSE for mediator \(j\) under
that permutation. This average depends on the chosen set \(\Pi\) and is used
only as a descriptive sensitivity summary. In the main text, we use the given mediator indexing as a reference ordering
to present the IPSE construction.

\section{Semi-Synthetic IHDP Data Generating Mechanism}
\label{app:ihdp-dgm}

We retain the observed IHDP treatment indicator \(A_i\in\{0,1\}\) and use the first ten baseline covariates as \(\bm Z_i=(Z_{i1},\ldots,Z_{i10})\).

The mediators are generated as
\[
\begin{aligned}
M_{i1}
&=
0.10 + 0.55 A_i + 0.20 Z_{i1} - 0.15 Z_{i2}
+ 1.10\,\varepsilon_{i1},\\
M_{i2}
&=
-0.15 + 0.60 A_i + 0.10 Z_{i4}
+ 1.10\,\varepsilon_{i2},\\
M_{i3}
&=
-0.70 + 0.65 A_i + 0.08 Z_{i5}
+ 1.10\,\varepsilon_{i3},
\end{aligned}
\]
where
\[
\varepsilon_{i1},\varepsilon_{i2},\varepsilon_{i3}
\overset{\mathrm{i.i.d.}}{\sim}N(0,1).
\]

The outcome is generated as
\[
Y_i = \mu_i+\sigma_i R_i,
\]
where
\[
\begin{aligned}
\mu_i
&=
1.00 + 0.98 M_{i1} + 0.08 Z_{i1} - 0.06 Z_{i2},\\
\log\sigma_i
&=
0.55 M_{i2},\\
R_i
&=
U_i+
10\,\operatorname{expit}\{5(M_{i3}+0.05)\}\,g_i,
\qquad U_i\sim N(0,1),
\end{aligned}
\]
and
\[
g_i
=
\operatorname{expit}\{5(Z_{i5}-\tau_G)\}
-
\frac1n\sum_{\ell=1}^n
\operatorname{expit}\{5(Z_{\ell5}-\tau_G)\}.
\]
Here \(\tau_G\) is the empirical 0.88 quantile of
\(\{Z_{i5}:i=1,\ldots,n\}\), and
\(\operatorname{expit}(u)=\{1+\exp(-u)\}^{-1}\). The centered gate \(g_i\) is
positive mainly for subjects with large \(Z_{i5}\) and slightly negative for
most others. Thus, larger values of \(M_{i3}\) amplify positive residual values
mainly in a small high-\(Z_{i5}\) subgroup, producing a heavier right tail and
greater tail asymmetry.

\section{Implementation Details}
\label{app:implementation-details}

Table~\ref{tab:generator_architecture} summarizes the generator architectures used for the main experiments. Unless otherwise stated, all generators are trained with the ES objective using Adam with learning rate \(5\times 10^{-4}\), a 20\% validation split, and early stopping. The ES-based loss uses \(L_{\mathrm{ES}}=20\) generator draws.

\begin{table}[h]
\centering
\caption{
Generator architectures used in the main experiments.
The notation $L\times H$ denotes $L$ hidden layers with hidden width $H$.
}
\label{tab:generator_architecture}
\small
\setlength{\tabcolsep}{4pt}
\renewcommand{\arraystretch}{1.08}
\resizebox{\columnwidth}{!}{%
\begin{tabular}{lcccc}
\toprule
\multirow{2}{*}{Experiment}
& \multicolumn{2}{c}{Noise dimension}
& \multicolumn{2}{c}{Network architecture} \\
\cmidrule(lr){2-3}\cmidrule(lr){4-5}
& $\varepsilon_M$ & $\varepsilon_Y$
& Mediator model & Outcome model \\
\midrule
Synthetic S1
& 4 & 8
& $5\times 64$ & $5\times 64$ \\
Synthetic S2
& 4 & 8
& $5\times 64$ & $5\times 64$ \\
IHDP
& 4 & 32
& $5\times 64$ & $7\times 128$ \\
NHANES
& 4 & 8
& $3\times 64$ & $3\times 64$ \\
\bottomrule
\end{tabular}%
}
\end{table}


\end{document}